\documentclass[letterpaper, 10 pt, conference]{ieeeconf}

\usepackage{mathtools}
\usepackage{amsmath,mleftright}
\usepackage{xparse}
\usepackage{comment}
\usepackage{microtype}

\usepackage{graphicx}%
\usepackage{amsmath}
\usepackage{amssymb}
\let\labelindent\relax
\usepackage{enumitem}
\usepackage{tabularx}
\usepackage[table]{xcolor}
\usepackage[noadjust]{cite} 
\usepackage{mathrsfs}

\usepackage{amsthm}

\usepackage[colorinlistoftodos]{todonotes}

\newtheorem{proposition}{Proposition}

\newtheorem{definition}{Definition}%

\usepackage[colorlinks=false,urlcolor=blue]{hyperref} 
\hypersetup{
    colorlinks=true,
    linkcolor=darkgray,
    filecolor=magenta,      
    urlcolor=cyan,
     citecolor    = darkgray
}

 	\definecolor{newgreen}{RGB}{0,127,0}
  \definecolor{newblue}{RGB}{0,0,127}

\begin{document}

\title{\LARGE \bf
Efficient Computation of Distance Functions\\for Navigation Vector Fields in Lie Groups
}

\author{%
Vinicius M. Gon\c{c}alves, João Baião,
Felipe Bartelt, Douglas G. Macharet, Gustavo M. Freitas,\\ Héctor Azpúrua, Luciano C. A. Pimenta 
\thanks{
}%
}

\maketitle

\maketitle
\begin{abstract} 
Vector-field-based methods are widely used for robot control and are often applied to the path-tracking problem. Some vector field approaches require repeatedly computing the distance between the robot configuration and the curve, as well as the corresponding closest point. Recently, vector fields have been extended to Lie Groups. In this case, this computation can be expensive, especially when performed at high control frequencies on embedded platforms. This paper proposes a method for efficiently computing the distance between a point and a curve represented as what is called a G-polynomial curve, which is a curve representation that generalizes polynomial curves to matrix Lie groups. The proposed approach exploits the structure of these curves to reduce the problem to a small number of polynomial root-finding computations. Simulation results show that the method significantly reduces computation time while maintaining accuracy compared to existing optimization-based approaches. Practical formulas are also provided for the case of the group SE(3), and the method is validated experimentally on a robotic manipulator. The methodology is implemented in a computational package, available online.
\end{abstract}

\section{Introduction}

Vector fields have been widely used for robot control \cite{goncalves2010, AdrianoIEEE, yao2021singularity, yao2022topological, chen2025nonsingular, zhou2025inverse, bartelt2026constructive, nunes2026safe}  for several reasons, including their simplicity and strong theoretical foundations. A common problem addressed by vector field methods is the \emph{path tracking problem}, defined as follows: consider $h \in \mathbb{R}^n$ as the configuration of a fully actuated robot, whose dynamics are given by $\dot{h} = u$, and let $\mathcal{C} \subseteq \mathbb{R}^n$ be a one-dimensional differentiable target curve. The objective is to design a configuration-feedback controller $u = \Psi(h)$ such that, for any initial condition $h$, the configuration converges to $\mathcal{C}$ and subsequently traverses it. Some vector field approaches to this problem rely on computing, at each time instant, the \emph{distance} $D(h)$ between the current configuration $h$ and the curve $\mathcal{C}$, as well as the corresponding closest point $h^*(h)$ on the curve. For example, in \cite{AdrianoIEEE, nunes2026safe}, the \emph{Euclidean distance} must be computed. 

More recently, the work in \cite{bartelt2026constructive} extended the method in \cite{AdrianoIEEE} to the matrix Lie group setting. This is of engineering interest because it allows one, for example, to control the position and orientation of rigid objects independently. This situation is applicable to omnidirectional drones \cite{HamandiOmni} or the end effector of a manipulator with six or more degrees of freedom. Furthermore, it could be useful in other less-known groups, but that are also relevant in robotics, as the Galilean Group \cite{kelly2024making,mahony2025galilean}. Nevertheless, in this case, one must compute a specific Lie group distance between an element (matrix) $H$ and a curve $\mathcal{C}$ interpreted as a continuous sequence of Lie group elements. Since this distance must be computed at high frequencies on hardware platforms that are often computationally limited, such as embedded processors on drones, there is a strong need for efficient computation. 

\begin{figure}[t]
    \centering
    \includegraphics[width=.85\linewidth]{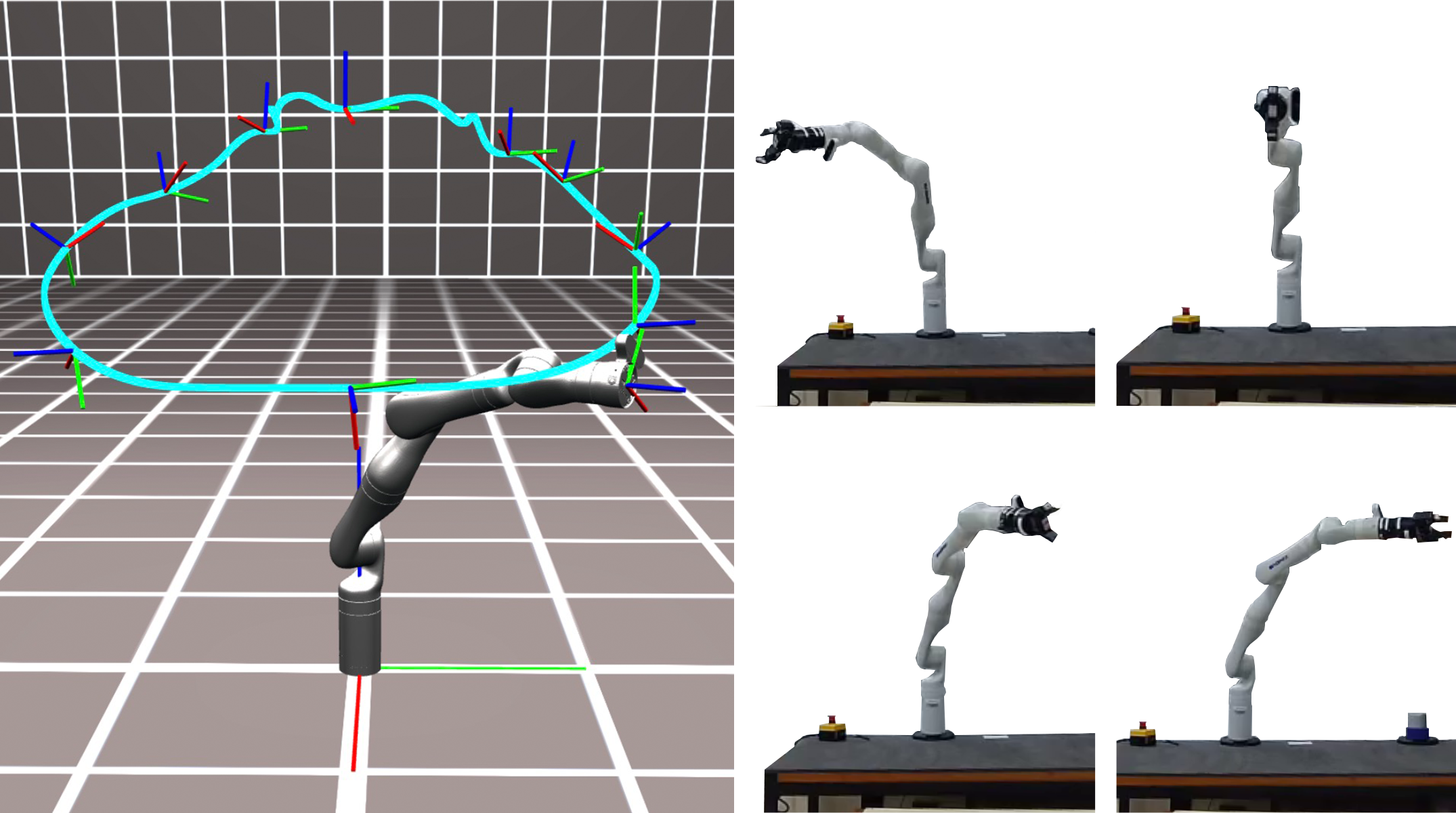}
    \caption{Closed-loop path tracking (cyan curve) experiment
using a 7-DOF Kinova Gen3 manipulator (right) to track a curve in the pose space.}
\vspace{-.7cm}
    \label{fig:introduction}
\end{figure}

More precisely, at each control loop one must solve a \emph{one-dimensional} optimization problem (since the curve $\mathcal{C}$ is one-dimensional). 
 One might be tempted to exploit the continuity of the problem, since from time $t$ to $t+\Delta t$ the element $H$ changes only slightly. However, although the distance $D(H)$ is guaranteed to be a continuous function of $H$, it is well known that the \emph{closest point map} $H^*(H)$ is, in general, not continuous. Discontinuities occur at points $H$ that are ``problematic'', that is, for which more than one closest point on the curve exists. Although reaching such points exactly is practically impossible (since they form a set of measure zero), in the vicinity of these points the closest point may change continuously but extremely rapidly. A classic example illustrating this behavior arises when $\mathcal{C}$ is the unit circle in $\mathbb{R}^2$. The center of the circle is a point of discontinuity of the closest point map. If one selects a point very close to the center and moves around it in a small circular motion, the closest point on $\mathcal{C}$ changes very rapidly. Therefore, continuity cannot be reliably exploited, since the system may always operate arbitrarily close to such problematic points. At each time step, we have to essentially solve a new optimization problem from scratch.

Thus, we face the situation in which, at each time step, we must reliably and quickly find the global minimum of a generally multimodal one-dimensional function $q:[0,K] \to \mathbb{R}$ over the bounded interval $[0,K]$. The situation described in \cite{bartelt2026constructive} is particularly challenging because evaluating $q(s)$ for a general matrix Lie group can be relatively expensive when compared to the particular case of $\mathbb{R}^n$ considered in \cite{AdrianoIEEE}, where the computation essentially reduces to a Euclidean distance. For example, in the case of the group $\textsl{SE}(3)$, which is of practical engineering interest, a single evaluation of $q(s)$ is roughly ten times slower than a single evaluation of $q(s)$ in the case $\mathbb{R}^3$.

The simplest alternative is a brute force approach: sample many points in $[0,K]$ and test each one of them for the minimum. This is too expensive. Thus, currently, perhaps the most suitable approach is to use the \emph{Piyavskii–Shubert algorithm}~\cite{shubert1972sequential}. Given a Lipschitz constant $L$ for $q(s)$ and a parameter $\delta$, the algorithm finds an interval of length at most $\delta$ that contains a global minimum $s^*$.

We propose an alternative to the Piyavskii–Shubert approach for the specific problem of computing the distance between a point $H$ and a curve $\mathcal{C}$, when this curve is what we call a \emph{$G$-polynomial curve}. This type of curve is a generalizations of piecewise polynomial curves in the Lie Group setting, and have a considerable flexibility to interpolate points $C_k$ while matching derivatives, creating a differentiable, one-dimensional curve. Similar curves were considered in other papers \cite{sefrant1998interpolation,watterson2016smooth,mueller2025poe}. The method builds on the observation that, in the particular case of the Euclidean group, if the curve $\mathcal{C}$ is represented by segments given as the image of polynomials of order $M$, then this distance can be computed by solving a small number of polynomial root-finding problems of degree $2M-1$. In particular, when $M=2$, which is sufficiently expressive to model many curves while guaranteeing once differentiability, the resulting problem can be solved analytically using the classical Cardano's formula.

Thus, for the more general matrix Lie group case, the contribution of this paper is
a method to exploit $G$-polynomial curves in order to efficiently solve an \emph{approximate} version of the original problem by reducing it to a polynomial root-finding problem. We derive and analyze properties of this approximation that make it useful in practice. This approximation can either be used directly within the vector field methodology in \cite{bartelt2026constructive} or serve as a very good initial estimate for solving the original problem. In the latter case, we show, with extensive simulations, that we can outperform the Piyavskii–Shubert's approach by achieving faster computation while maintaining accuracy with respect to the true solution. Specifically, by taking the median speedup as the reference, the method provides computational speedups of values as high as $5 \times$  for a small  number of piecewise segments $K$. Speedups are observed up to $K=80$, with only 0.605\% of the total samples presenting an error on the parametrization $s \in [0,K]$ above 1.0\%. For the special case in which the group is $\textsl{SE}(3)$, useful in many robotics applications, practical formulas are presented. Furthermore, we also tested the vector field methodology in \cite{bartelt2026constructive} using this distance-computation technique on a real platform, a robotic manipulator (Fig.~\ref{fig:introduction}). Finally, the algorithm is implemented in the package [REMOVED FOR ANONYMITY], available online.

\section{Methodology}
\subsection{Preliminaries}

Let $G$ be an $n$-dimensional matrix Lie group and $\frak{g}$ its Lie algebra. The \emph{exponential map} $\exp: \frak{g} \to G$ is defined as
$\exp(A) = \sum_{i=0}^{\infty} \frac{A^i}{i!}$. Given $A \in \frak{g}$ and $H \in G$, let $\mathcal{T}_{H}G$ denote the tangent space of $G$ at the point $H$, and let $D\exp_{A}: \frak{g} \to \mathcal{T}_{\exp(A)}G$ be the linear map representing the differential of $\exp$ at the point $A$:
\begin{equation}
  D\exp_{A}[E] \triangleq \left. \frac{d}{dt} \exp(A+Et)  \right|_{t=0}.
\end{equation}

We define the operator \textls{vec}: $\frak{g} \to \mathbb{R}^n$ that maps the matrix $A$ to an $n$-dimensional vector representing its entries, and the inverse operator \textls{vec}$^{-1}: \mathbb{R}^n \to \frak{g}$. For a given group $G$, this operator \textsl{vec}$(\cdot)$ is not unique and depends on the choice of basis for the Lie algebra.

If $A$ is a square matrix, $\textsl{Tr}(A) \triangleq \sum_i A_{ii}$ denotes its trace, and $\|B\|$ denotes the \emph{Frobenius norm}, defined as
$\|B\| \triangleq \sqrt{\textsl{Tr}(B^{\top}B)}$.

Let $G^{\circ} \subseteq G$ be the set of matrices $H \in G$ in which no eigenvalue is a real negative number, and let $\frak{g}^{\circ} \subseteq \frak{g}$ be the set of matrices $A \in \frak{g}$ whose eigenvalues have imaginary parts $\omega$ satisfying $-\pi < \omega < \pi$ (the \emph{open strip}).

We adopt, with adaptations, the definition of a \emph{special exponential} matrix Lie group from \cite{bartelt2026constructive}. A matrix Lie group $G$ is said to be \emph{special exponential} if the map $\exp: \frak{g}^{\circ} \to G^{\circ}$ is \emph{bijective}. Not all groups satisfy this property, but important groups such as $\mathbb{R}^n$, $\textsl{SO}(n)$, and $\textsl{SE}(n)$ do, in addition to  less known but also robotic-relevant as the Gallilean Group $\textsl{SGal}(3)$ \cite{kelly2024making,mahony2025galilean}. For special exponential groups, one can define a map $\log: G^{\circ} \to \frak{g}^{\circ}$ such that $\exp(\log(H)) = H$ for any $H \in G^{\circ}$ and $\log(\exp(A)) = A$ for any $A \in \frak{g}^{\circ}$. This map is called the \emph{principal logarithm} \cite{higham2008functions}.

To make the exposition smoother, in this paper, whenever Lie Group/Algebra is being discussed, members of the Lie Group are represented by upper case \emph{consonants} (e.g., $H, C$...), whereas members of the Lie Algebra are represented by upper case \emph{vowels} (e.g., $A$, $E$, ...). The only exception to this rule is the matrix $I_{n \times n}$, which denotes a $n \times n$ identity matrix, which is the identity element of $G$ when it is a $n$ dimensional Matrix group.

\subsection{The Euclidean case}
\label{subs:euclidean}
To illustrate the methodology, we begin by considering a special case in which $G=\mathbb{R}^n$. We assume that the target curve $\mathcal{C} \subseteq \mathbb{R}^n$ is formed by $K$ segments $\mathcal{C}_k$, each of them being the image of an $M$-th order polynomial $p_k:[0,1]\to \mathbb{R}^n$. In this case, the optimization problem in \cite{AdrianoIEEE}
\begin{equation}
\label{eq:opadriano}
    D(h) = \min_{c \in \mathcal{C}} \|c-h\|
\end{equation}
\noindent can be reduced to solving the problem for each segment and then taking the minimum:
\begin{equation}
    D(h) = \min_{k \in \{1,\ldots,K\}} \min_{s \in [0,1]} \|p_{k}(s)-h\|.
\end{equation}

A key observation is that, for a polynomial $p_k(s)$, $\|p_{k}(s)-h\|^2$ is another polynomial, of degree $2M$, in $s$. Thus, finding $s$ that minimizes $\|p_{k}(s)-h\|$ on $[0,1]$ (equivalently, minimizing $\|p_{k}(s)-h\|^2$) reduces to finding the roots of a polynomial (the derivative of the original polynomial) of degree $2M-1$, together with checking the boundary points $s=0$ and $s=1$. Efficient algorithms exist for computing the roots of such polynomials when $M$ is not large. In particular, when $M=2$, the problem can be solved analytically since $2M-1=3$, allowing the use of Cardano's formula.

For general matrix (special exponential) Lie groups, we would like to obtain a similar formulation in which the minimizer can be computed efficiently. This will be discussed in the next subsection.

\subsection{Extension to the Lie Group case}
Let $G$ be a Special Exponential matrix Lie group and $\frak{g}$ its Lie algebra. Let $\mathcal{C}$ be a closed, one-dimensional, once-differentiable curve in $G$. The counterpart of \eqref{eq:opadriano} in \cite{bartelt2026constructive}, which generalizes \cite{AdrianoIEEE}, is the following,  for $H \in G$:
\begin{equation}
\label{eq:D}
    D(H) \triangleq \min_{C \in \mathcal{C}} \|\log(H^{-1}C)\|.
\end{equation}
We start by considering a special class of curves $\mathcal{C}$.

\begin{definition}
\textbf{(G-Polynomial)} Curve $\mathcal{C}$ is called \emph{$G$-polynomial} if it is the union of $K$ (finite) segments $\mathcal{C}_k$, each of which is the image of a function $P_k:[0,1]\to G$ of the form
\begin{equation}
\label{eq:paramP}
    P_k(s) \triangleq P_{k0}\exp\underbrace{\left( \sum_{m=1}^M E_{km}s^m \right)}_{\triangleq E_k(s)} = P_{k0}\exp\big(E_k(s)\big) ,
\end{equation}
\noindent where $E_{km} \in \frak{g}$ and $P_{k0} \in G$, and such that the values and derivatives match at the endpoints, i.e.,
\begin{eqnarray}
    P_k(1) = P_{k+1}(0) \ , \ P'_k(1) = P'_{k+1}(0)
\end{eqnarray}
\noindent for $k=1,\ldots,K-1$. If the curve is closed, it must also hold that $P_K(1) = P_1(0)$ and $P'_K(1) = P'_1(0)$. %
\end{definition}

The concept of $G$-Polynomial is not new. Many works propose similar,  equivalent, or even more general, curves. This is often the case in the setting of \emph{interpolation} in Lie Groups, because such a definition is amenable to interpolating points $C_k \in G$ while matching derivatives.  One recent example is \cite{mueller2025poe}, for general Lie Groups. In \cite{watterson2016smooth}, a very close definition was made for the special case in which $G=\textsl{SE}(3)$. Nevertheless, this  exact definition is adequate for our purpose.

Although having higher-order derivatives for $\mathcal{C}$ may be desirable, the vector field in \cite{bartelt2026constructive}, where $D$ is used, requires only that the first derivative be continuous. This motivates the definition of a $G$-polynomial curve with this minimal requirement.

The expression in \eqref{eq:paramP} can be viewed as the counterpart of a polynomial in Lie groups. Indeed, when $G = \mathbb{R}^n$, where $\mathbb{R}^n$ is represented by homogeneous transformation matrices with identity rotation, $P_k(s)$ becomes a standard polynomial in $s$, since in this case $\exp(U) = I+U$ for any $U$. In Section~\ref{subs:approximateC}, we will show how to construct a $G$-polynomial curve $\mathcal{C}$ from samples of $\mathcal{C}$ (a $G$-polynomial interpolation) by solving a linear system.

When $\mathcal{C}$ is $G$-polynomial, we approximate $\hat{D}(H,C) \triangleq \|\log(H^{-1}C)\|$ with an expression easier to optimize. When $G=\mathbb{R}^n$, this reduces to the situation discussed in Subsection \ref{subs:euclidean}, and the approximation becomes exact. To do so, we study a specific function $f$ and its approximation $\tilde{f}$.

\subsection{The functions $f$ and $\tilde{f}$}
Consider the set
\begin{equation}
    (\frak{gg})^{\circ} \triangleq \{(A,E) \in \frak{g} \times \frak{g} \ | \ \exp(-A)\exp(E) \in G^{\circ}\}
\end{equation}
\noindent and the function $f:(\frak{gg})^{\circ} \to \frak{g}$:
\begin{equation}
    f(A,E) \triangleq \log\big(\exp(-A)\exp(E)\big).
\end{equation}

Consider $A \in \frak{g}$ such that $(A,0) \in (\frak{gg})^{\circ}$. One can easily see that this is equivalent to $-A \in \frak{g}^{\circ}$. Since $\frak{g}$ is a vector space, this is also equivalent to $A \in \frak{g}^{\circ}$. The function $f:(\frak{gg})^{\circ} \to \frak{g}$ is smooth on its domain, and we can write a first-order expansion around $E=0$ provided that $A \in \frak{g}^{\circ}$:
\begin{equation}
\label{eq:f}
  f(A,E) = -A + \mathbb{L}_{A}[E] + \mathcal{O}(\|[E,A]\|).
\end{equation}

\noindent Here $[E,A] \triangleq EA-AE$ denotes the \emph{commutator}, $f(A,0)=-A$ is the zeroth-order term, $\mathbb{L}_{A}[E]$ is the first-order term, and $\mathcal{O}(\|[E,A]\|)$ represents higher-order terms that vanish as $[E,A] \rightarrow 0$ (which is a weaker condition than $\|E\| \rightarrow 0$). %

\begin{definition}
Given $A \in \frak{g}^{\circ}$, we define the linear operator $\mathbb{L}_A:\frak{g}\to\frak{g}$ by
\begin{equation}
\mathbb{L}_A[E] \triangleq \left.\frac{d}{dt} f(A,Et)\right|_{t=0}
= \left.\frac{d}{dt}\log\big(\exp(-A)\exp(tE)\big)\right|_{t=0}.
\end{equation}
\end{definition}

The computation of this term in the special and important case $G=\textsl{SE}(3)$ will be discussed in Subsection \ref{subs:computinglA}. Nevertheless, defining $\tilde{f}:\frak{g}^{\circ}\times\frak{g}\to\frak{g}$ as
\begin{equation}
    \tilde{f}(A,E) \triangleq -A + \mathbb{L}_{A}[E],
\end{equation}
\noindent we have $f(A,E) \approx \tilde{f}(A,E)$ provided that $[E,A]$ is small (see Subsection \ref{subs:validity} for a discussion on the validity of this approximation).
We now establish several properties of the map $\mathbb{L}_A$ that will be useful for our purposes.

\begin{proposition}
\label{ref:propmain}
For $A \in \frak{g}^{\circ}$, $E_1,E_2,E \in \frak{g}$ and $\alpha_1,\alpha_2 \in \mathbb{R}$, the following properties of $\mathbb{L}_A$ hold:
\begin{itemize}
    \item (1): $\mathbb{L}_A[\alpha_1E_1+\alpha_2E_2] = \alpha_1 \mathbb{L}_A[E_1] + \alpha_2 \mathbb{L}_A[E_2]$;
    \item (2): $\mathbb{L}_A[A] = A$;
    \item (3): $\mathbb{L}_A[E] = A$ implies $E=A$.
\end{itemize}
\end{proposition}

\begin{proof}
The first property follows from the fact that $\mathbb{L}_A$ is the first-order term in the expansion of $f(A,E)$ around $E=0$, and therefore it must be linear in $E$.

The second property follows directly from the definition:
\begin{equation}
\left.\frac{d}{dt}\log\big(\exp(-A)\exp(tA)\big)\right|_{t=0}
= \left.\frac{d}{dt}\big((t-1)A\big)\right|_{t=0} = A.
\end{equation}
Here we used the identity $\exp(-A)\exp(tA)=\exp((t-1)A)$ and the fact that $\log(\exp(U))=U$ whenever $U\in\frak{g}^{\circ}$. Note that if $A\in\frak{g}^{\circ}$, then $(t-1)A\in\frak{g}^{\circ}$ for any $t<1$, and since the limit is taken as $t\rightarrow0$, this condition is satisfied.

For the third property, we show that the map $\mathbb{L}_A$ has a trivial kernel, i.e., $\mathbb{L}_A[E]=0$ implies $E=0$. Together with properties (1) and (2), this implies (3).

To prove that the kernel is trivial, note that $\exp(f(A,Et))=\exp(-A)\exp(Et)$. Taking the derivative with respect to $t$ and evaluating at $t=0$ on both sides yields
$D\exp_{-A}[\mathbb{L}_A[E]]=\exp(-A)E$. Suppose that $\mathbb{L}_A[E]=0$. Then $D\exp_{-A}[0] = \exp(-A)E$. Since any differential is linear, $D\exp_{-A}[0]=0$, and therefore $\exp(-A)E=0$. Because $\exp(-A)$ is invertible, this implies $E=0$, proving that the kernel is trivial.
\end{proof}

To compute $\mathbb{L}_A$, we need another definition. Note that, for $A \in \frak{g}$, the map $\exp(-A)D\exp_A: \frak{g} \to \frak{g}$ is linear. %

\begin{definition} \label{def:leftjac}
Assume that $G$ is $n$-dimensional and let $A \in \frak{g}$. Consider a vectorization procedure \textsl{vec}. Define $\mathbb{I}_A \in \mathbb{R}^{n \times n}$ as the matrix satisfying
\begin{equation}
    \textsl{vec}\big(\exp(-A)D\exp_A[E]\big) = \mathbb{I}_A\textsl{vec}(E)
\end{equation}
\noindent for any $E \in \frak{g}$. Explicitly, let $e_i$ be the $i^{th}$ column of the identity matrix $I_{n \times n}$. Then the $i^{th}$ column of $\mathbb{I}_A$ is $\textsl{vec}\big(\exp(-A)D\exp_A[\textsl{vec}^{-1}(e_i)]\big)$.
\end{definition}

The matrix $\mathbb{I}_A$ corresponds to what is called the \emph{left Jacobian} of the Lie group exponential. Note that the definition of $\mathbb{I}_A$ depends implicitly on the choice of the vectorization operator \textsl{vec}.

\begin{proposition} \label{prop:prop2}
If $A \in \frak{g}^{\circ}$, then $\mathbb{I}_A$ is invertible.
\end{proposition}

\begin{proof}
It suffices to prove that the linear map $D\exp_A$ is invertible when $A \in \frak{g}^{\circ}$. This is equivalent to showing that it has a trivial kernel, i.e., if $A \in \frak{g}^{\circ}$ and $D\exp_A[E]=0$, then $E=0$.

Let $E \in \frak{g}$. For $t$ sufficiently close to $0$, we have $A+Et \in \frak{g}^{\circ}$ and hence $\exp(A+Et) \in G^{\circ}$. The map $\exp:\frak{g}^{\circ}\to G^{\circ}$ is injective \cite{higham2008functions}, and therefore $\log(\exp(A+Et))=A+Et$ for $t$ sufficiently close to $0$. Differentiating both sides with respect to $t$ and evaluating at $t=0$, and using the differentiability of the logarithm in this region \cite{higham2008functions}, we obtain
$D\log_{\exp(A)}[D\exp_A[E]] = E$.

Now suppose that $D\exp_A[E]=0$. Then $D\log_{\exp(A)}[0]=0=E$, where we used the linearity of $D\log$. %
\end{proof}

With this result, we can compute $\mathbb{L}_A$.

\begin{proposition} \label{prop:prop3}
For $A \in \frak{g}^{\circ}$ and $E \in \frak{g}$,
\begin{equation}
    \mathbb{L}_A[E] = \textsl{vec}^{-1}\Big(\mathbb{I}_{-A}^{-1}\textsl{vec}(E)\Big).
\end{equation}
\end{proposition}

\begin{proof}
Note that
\begin{equation}
    \exp(f(A,Et)) = \exp(-A)\exp(tE).
\end{equation}
Taking the derivative with respect to $t$ on both sides and evaluating at $t=0$, we obtain
\begin{equation}
    D\exp_{-A}\big(\mathbb{L}_A[E]\big) = \exp(-A)E.
\end{equation}
Left-multiplying both sides by $\exp(A)$ and applying \textsl{vec}, and using the definition of $\mathbb{I}_{-A}$, yields
\begin{equation}
    \mathbb{I}_{-A}\textsl{vec}\big(\mathbb{L}_A[E]\big) = \textsl{vec}(E).
\end{equation}
From Proposition \ref{prop:prop2}, the matrix $\mathbb{I}_{-A}$ is invertible when $A \in \frak{g}^{\circ}$ (note that $A \in \frak{g}^{\circ}$ implies $-A \in \frak{g}^{\circ}$). The desired result follows by multiplying both sides on the left by the inverse of $\mathbb{I}_{-A}$ and applying the inverse vectorization operator.
\end{proof}

\subsection{Approximating the distance}
Consider a curve $\mathcal{C}$ that is $G$-polynomial. Then
\begin{align}
  D(H) & =  \min_{C \in \mathcal{C}} \|\log(H^{-1}C)\|  \nonumber \\
  & =   \min_{k \in \{1,...,K\}}\min_{s \in [0,1]} \|\log\big(H^{-1}P_k(s)\big)\| \nonumber \\
  & =   \min_{k \in \{1,...,K\}}\min_{s \in [0,1]} \|\log\big(H^{-1}P_{k0}P_{k0}^{-1}P_k(s)\big)\|.  \nonumber
\end{align}
\noindent Here we multiplied by $P_{k0}P_{k0}^{-1}$, which is the identity matrix, for reasons that will become clear shortly.

Assume that $P_{k0}^{-1}H \in G^{\circ}$. Furthermore, assume that for all $k$ and all $s \in [0,1]$, $E_k(s) \in \frak{g}^{\circ}$. These assumptions will be discussed in Subsection \ref{subs:validity}. Under these conditions, $\log(P_{k0}^{-1}P_k(s)) = E_k(s)$ and we can define $A_k \triangleq \log(P_{k0}^{-1}H) \in \frak{g}^{\circ}$. Using function $f$ in \eqref{eq:f}, we obtain
\begin{eqnarray}
  D(H) =   \min_{k \in \{1,...,K\}}\min_{s \in [0,1]} \|f\big(A_k,E_k(s)\big)\|.
\end{eqnarray}

We now use the approximation $f(A,E) \approx \tilde{f}(A,E)$, assuming that the approximation is sufficiently accurate (see Subsection \ref{subs:validity}). The quality of this approximation depends, among other factors, on $E_k(s)$ being ``small''. This is reasonable if the points $C_k$ and $C_{k+1}$ are sufficiently close to each other, since $E_k$ represents a type of displacement between these two elements. This can be controlled by choosing sufficiently close samples when constructing the $G$-polynomial curve $\mathcal{C}$. This flexibility was made possible by the multiplication by $P_{k0}P_{k0}^{-1}$, which allowed us to separate $E_k$, the \emph{displacement}, from $P_{k0}^{-1}H$.

Thus, we must solve the optimization problem
\begin{equation}
\label{eq:minp}
    \min_{s \in [0,1]} \|\tilde{f}(A_k,E_k(s))\|
\end{equation}
\noindent for $k \in \{1,...,K\}$. Since $\tilde{f}(A,E)$ is affine in $E$ (a consequence of property (1) of Proposition \ref{ref:propmain}) and $E_k(s)$ is a polynomial of degree $M$ in $s$, the function $\|\tilde{f}(A_k,E_k(s))\|$ is the square root of a polynomial of degree $2M$. Therefore, the minimization can be performed by finding the roots of a polynomial of degree $2M-1$, selecting those that lie in the interval $[0,1]$, evaluating the polynomial at these points, and also at the boundaries $s=0$ and $s=1$.
With this, given a $G$-polynomial curve $\mathcal{C}$, we define an \emph{approximate} version of $D(H)$, namely the function $\tilde{D}: G \to \mathbb{R}$ defined as
\begin{equation}
\label{eq:Dtilde}
    \tilde{D}(H)  \triangleq \min_{k \in \{1,...,K\}} \min_{s \in [0,1]} \|\tilde{f}(A_k,E_k(s))\|.
\end{equation}

To ensure that this approximation behaves as a proper distance-like metric, we have the following result.

\begin{proposition} \label{prop:prop4}
The expression $\tilde{D}(H)$ is nonnegative, and it is equal to zero if and only if $H \in \mathcal{C}$.
\end{proposition}

\begin{proof}
The result of the minimization in \eqref{eq:minp} is clearly nonnegative due to the Frobenius norm.
Furthermore, from properties (2) and (3) in Proposition \ref{ref:propmain}, the expression $\|\tilde{f}(A_k,E_k(s))\|$ is zero if and only if there exist $k \in \{1,...,K\}$ and $s \in [0,1]$ such that $E_k(s) = A_k$. In this case,
$P_k(s) = P_{k0}\exp(E_k(s)) = H$, which implies that $H$ belongs to one of the segments $\mathcal{C}_k$, and therefore $H \in \mathcal{C}$.
\end{proof}
\subsection{Algorithm for computing $D$ from the approximation}
\label{subs:alg}
For simplicity, we assume that the curve is a closed $G$-polynomial curve with $K$ segments and polynomials of order $M=2$. Given such a curve, we want to compute $C^*$ that minimizes $D$ in \eqref{eq:D} for a given $H$ by means of the proposed approximation.
Thus, we consider a parametrization $P(s)$ of $\mathcal{C}$ obtained from the piecewise $G$-polynomials $P_k$. That is, a function $P:[0,K] \to\mathcal{C}$ defined by
\begin{equation}
    P(s) = P_{\textsl{ceil}(s)}\big(\textsl{frac}(s)\big)
\end{equation}
\noindent for $0 < s \leq K$ and $P(0)=C_1$, where $\textsl{ceil}$ denotes the ceiling operator and $\textsl{frac}$ the fractional part. Then the problem in~\eqref{eq:D} of finding $C^*$ can be recast as finding $s^*\in[0,K]$ that minimizes $q(s)\triangleq\|\log(H^{-1}P(s))\|^2$.

The first part of the algorithm is to solve an \emph{approximate} version of this problem by finding minimizers $k$ and $s$ of $\tilde{D}$ in \eqref{eq:Dtilde}. Thus, assuming that we have the initial points $C_k$ and the coefficients $E_{k1},E_{k2}$ for each segment, we perform the following steps for $k\in\{1,\ldots,K\}$:

\begin{itemize}
    \item Compute $A_k=\log(C_k^{-1}H)$ (the formula for the principal logarithm in $\textsl{SE}(3)$ can be found in \cite{barfoot2017state});
    \item Compute $U_{k1}=\mathbb{L}_{A_k}[E_{k1}]$ and $U_{k2}=\mathbb{L}_{A_k}[E_{k2}]$ (see Subsection \ref{subs:computinglA} for the $\textsl{SE}(3)$ case);
    \item Compute the coefficients of the polynomial
    \begin{equation}
    q_k(s)\triangleq\|-A_k+U_{k1}s+U_{k2}s^2\|^2
    \end{equation}
    \noindent which can be written as $q_k(s)=b_{k0}+b_{k1}s+b_{k2}s^2+b_{k3}s^3+b_{k4}s^4$, with 
    \begin{eqnarray}
       && b_{k0}=\|A_k\|^2 \ , \ b_{k1}=-2\textsl{Tr}(A_k^{\top}U_{k1}) \nonumber \\
        && b_{k2}{=}-2\textsl{Tr}(A_k^{\top}U_{k2}){+}\|U_{k1}\|^2 \ , \ \nonumber \\
        && b_{k3}{=}2\textsl{Tr}(U_{k1}^{\top}U_{k2}) \ , \ 
         b_{k4}=\|U_{k2}\|^2.
    \end{eqnarray}
    Note that if $X,Y\in\mathbb{R}^{n \times n}$, $\textsl{Tr}(X^{\top}Y)$ can be computed in $\mathcal{O}(n^2)$ operations, since $\textsl{Tr}(X^{\top}Y) = \sum_{ij} X_{ij}Y_{ij}$.
    \item Compute the roots of $q_k'(s)=4b_{k4}s^3+3b_{k3}s^2+2b_{k2}s+b_{k1}$. Consider those that lie in the interval $s\in[0,1]$. Evaluate $q_k(s)$ at these points and also at the boundaries $s=0$ and $s=1$. Select the value $s$ that yields the smallest result, which we denote by $\tilde{s}^*_k$;
    \item Return the triple $\{k,\tilde{s}^*_k,q_k(\tilde{s}^*_k)\}$.
\end{itemize}

The value of $k$ that yields the smallest $q_k(\tilde{s}^*_k)$ is then selected, together with its corresponding $\tilde{s}^*_k$. This value is mapped to the interval $[0,K]$, which is the domain of the function $q(s)$ that we ultimately want to minimize, by computing $\tilde{s}^*=k-1+\tilde{s}^*_k$. This provides a first estimate of the true minimizer $s^*$ of $q$, which is expected to be close to the actual solution (see Subsection \ref{subs:validity}).

A second step is subsequently performed to refine this estimate.
A simple brute-force strategy is adopted: $q'(\tilde{s}^*)$ is approximated using central differences and, according to its sign, the search proceeds from $\tilde{s}^*$ in the opposite direction using steps of size $\epsilon$, where $\epsilon$ is a small number, until no further reduction in $q(s)$ is achieved.

\section{Practical Considerations}
\subsection{Validity of the approximation}
\label{subs:validity}

Several assumptions contribute to the validity of the approximation from $D(H)$ in \eqref{eq:D} to $\tilde{D}(H)$ in \eqref{eq:Dtilde}. First, \eqref{eq:f} shows that the approximation depends on the commutator $[A,E]$ being small. In our case, this condition should hold for the pair $(k,s_k)$ that minimizes \eqref{eq:Dtilde}, requiring the commutator $[A_k,E_k(s_k)]$ to be small, where $A_k=\log(C_k^{-1}H)$ and $E_k(s)=\sum_{m=1}^M E_{km}s^m$.

Three factors contribute to this condition. Each of them, individually, improves the accuracy of the approximation:

\begin{itemize}
    \item In Abelian groups (e.g., $\mathbb{R}^n$), commutators vanish trivially. Hence, the approximation is exact, and $\mathbb{L}_A[E]=E$. In particular, when the method is applied to the setting of~\cite{AdrianoIEEE}, it yields the exact solution.
    
    \item As the points $C_k$ become closer to each other, the corresponding segments become shorter. Consequently, the matrices $E_{km}$ tend to become smaller, and therefore $E_k(s)$ becomes small for \emph{any $s\in[0,1]$}, reducing the magnitude of the commutator.
    
    \item As the point $H$ approaches $\mathcal{C}$, for the values $k$ and $s_k$ that minimize \eqref{eq:Dtilde}, we have $A_k \approx E_k(s_k)$, which reduces the commutator. In fact, according to Proposition \ref{prop:prop4}, the approximation becomes exact when $H \in \mathcal{C}$.
\end{itemize}

Throughout the paper, we also assumed that for all $k$, $H^{-1}C_k \in G^{\circ}$. In many important groups, such as $\textsl{SE}(m)$ and $\textsl{SO}(m)$, the set of \emph{problematic points} $H \in G$ --- that is, those for which there exists a $k$ such that $H^{-1}C_k \notin G^{\circ}$ --- has measure zero. For example, in $\textsl{SE}(m)$ and $\textsl{SO}(m)$ this corresponds to the case where $H^{-1}C_k$ represents a rotation of angle $\pm\pi$. In the rare cases in which this situation occurs (in practice, it seldom occurs because the set has measure zero), one may compute the minimizer $s$ for the true function $q_k(s)=\big\|\log\big(H^{-1}P_k(s)\big)\big\|^2$ using another method, such as the Piyavskii–Shubert algorithm.

\subsection{Creating $G$-polynomial curves from samples}
\label{subs:approximateC}
Many papers exist for Lie Group interpolation \cite{sefrant1998interpolation,watterson2016smooth,mueller2025poe}, considering many constraints as matching points $C_k$ and specified high-order derivatives at each point. We consider here a particular case which is enough for the main motivation of this paper (Lie Group vector fields, as \cite{bartelt2026constructive}).

Given $K+1$ sample points $C_k \in G$, we want to interpolate them with a $G$-polynomial curve while having matching first derivative at the endpoints of the segments. We do not prescribe any specific derivative: we just want them to be equal. Furthermore, we assume:

\begin{itemize}
    \item The curve is closed, hence $C_{K+1} = C_1$. This is for the sake of simplicity: the methodology can be easily adapted to non-closed curves. Thus, in this Subsection, everytime the index $K+1$ appears, it should be reinterpreted as $1$ (e.g., $e_{K+1} = e_1$).
    \item The points $C_k$ are sufficiently close so that $C_k^{-1}C_{k+1} \in G^{\circ}$ for $k \in \{1,...,K\}$.
\end{itemize}

Since we require continuity between segments and continuity of the first derivative, it is sufficient to use second-order $G$-polynomials, i.e., we select $M=2$. We then impose
\begin{equation}
    P_k(0) = C_k \ , \ P_k(1) = C_{k+1}
\end{equation}
\noindent for $k \in \{1,...,K\}$. Using \eqref{eq:paramP}, this can be written as
\begin{equation}
    \exp(E_k(0)) = P_{k0}^{-1}C_k \ , \ \exp(E_k(1)) = P_{k0}^{-1}C_{k+1}.
\end{equation}

Since $E_k(0)=0$, it follows that $P_{k0}=C_k$ for $k \in \{1,...,K\}$, and the first equation is automatically satisfied. We therefore need to satisfy $\exp(E_k(1)) = P_{k0}^{-1}C_{k+1} = C_k^{-1}C_{k+1}$ for $k \in \{1,...,K\}$, which requires determining the coefficients $E_{km}$. Since we assumed that $C_k^{-1}C_{k+1} \in G^{\circ}$ for all $k$, we can apply the logarithm to both sides, yielding
\begin{equation}
\label{eq:valcond}
 E_k(1) = E_{k1}+E_{k2} = \log(C_k^{-1}C_{k+1})
\end{equation}
\noindent for $k \in \{1,...,K\}$. These equations are \emph{linear} in the coefficients $E_{km}$. They provide $K$ equations for the $2K$ unknown coefficients, so $K$ additional equations are needed. These can be obtained by enforcing continuity of the derivatives:
\begin{equation}
\label{eq:dercond}
    P'_k(1) = P'_{k+1}(0)
\end{equation}
\noindent for $k=1,...,K$. Let $A_k \triangleq \log(C_k^{-1}C_{k+1})$ for $k=1,...,K$. Using the differential map, \eqref{eq:dercond} reduces to
\begin{equation}
    C_k D\exp_{A_k}[E_{k1}+2E_{k2}] = C_{k+1}D\exp_{0}[E_{k+1,1}].
\end{equation}

Since $D\exp_{0}[E]=E$, we obtain
\begin{eqnarray}
\label{eq:dercond2}
 && C_k D\exp_{A_k}[E_{k1}+2E_{k2}] = C_{k+1}E_{k+1,1} \rightarrow  \nonumber \\
 && \exp(-A_k) D\exp_{A_k}[E_{k1}+2E_{k2}] = E_{k+1,1}
\end{eqnarray}
\noindent for $k=1,...,K$.

The map $\exp(-A)D\exp_A:\frak{g}\to\frak{g}$ is related to the matrix $\mathbb{I}_A$ defined in Definition \ref{def:leftjac}. Applying \textsl{vec} to \eqref{eq:valcond} and \eqref{eq:dercond2}, and using Definition \ref{def:leftjac}, we obtain
\begin{eqnarray}
\label{eq:setofcond}
   && \textsl{vec}(E_{k1})+\textsl{vec}(E_{k2}) = \textsl{vec}(A_k) \nonumber \\ 
   && \mathbb{I}_{A_k}\big(\textsl{vec}(E_{k1})+2\textsl{vec}(E_{k2})\big) = \textsl{vec}(E_{k+1,1})
\end{eqnarray}
\noindent for $k=1,...,K$.

From the first equation we obtain $\textsl{vec}(E_{k2})=\textsl{vec}(A_k)-\textsl{vec}(E_{k1})$. Substituting this into the second equation yields equations involving only $E_{k1}$:
\begin{eqnarray}
    \mathbb{I}_{A_k}\textsl{vec}(E_{k1}) + \textsl{vec}(E_{k+1,1}) = 2\mathbb{I}_{A_k}\textsl{vec}(A_k).
\end{eqnarray}

Note that
\begin{equation}
  \mathbb{I}_{A_k}\textsl{vec}(A_k) = \textsl{vec}\Big(\exp(-A_k)D\exp_{A_k}[A_k]\Big).
\end{equation}

If $A$ and $E$ commute, then $\exp(-A)D\exp_A[E]=E$, which follows directly from the definition of $D\exp_A$. Hence $\mathbb{I}_{A_k}\textsl{vec}(A_k)=\textsl{vec}(A_k)$. Defining $e_k \triangleq \textsl{vec}(E_{k1})$ and $a_k \triangleq \textsl{vec}(A_k)$, we obtain the following system of $K$ linear equations in $K$ variables:
\begin{equation}
\mathbb{I}_{A_k}e_k + e_{k+1} = 2a_k
\end{equation}
\noindent for $k=1,...,K$. It can be verified by simple substitution that this system is solvable if the matrix
\begin{equation}
    I_{n \times n} + (-1)^{K+1}\mathbb{I}_{A_K}\mathbb{I}_{A_{K-1}}\cdots\mathbb{I}_{A_2}\mathbb{I}_{A_1}
\end{equation}
\noindent is invertible. Once $E_{k1}=\textsl{vec}^{-1}(e_k)$ is obtained, the remaining coefficients follow from $E_{k2}=A_k-E_{k1}$. Note that this linear system can be solved in $\mathcal{O}(n^3K)$ operations, in which $n$ is the dimension of the group.

\subsection{Computing  $\mathbb{I}_A$ and $\mathbb{L}_A$ in the case of \textsl{\textsl{SE}}(3)}
\label{subs:computinglA}

Now consider the case $G=\textsl{SE}(3)$ and $\frak{g}=\frak{se}(3)$. According to Proposition \ref{prop:prop3}, to compute $\mathbb{L}_A$, we need to compute $\mathbb{I}_A$, and thus this section will focus on this. The formulas in this section come mainly from classical Jacobian formulae for $\textsl{SE}(3)/\textsl{SO}(3)$ in \cite{barfoot2017state}.

To define the matrix $\mathbb{I}_A$, the operator \textsl{vec} must first be specified. Observe that any $A\in\frak{se}(3)$ can be written as
\begin{equation}
A=
\left(
\begin{array}{cccc}
0 & -\alpha_z & \alpha_y & a_x \\
\alpha_z & 0 & -\alpha_x & a_y \\
-\alpha_y & \alpha_x & 0 & a_z \\
0 & 0 & 0 & 0
\end{array}
\right).
\end{equation}

Hence, define the operator \textls{vec}:\,$\frak{se}(3)\to\mathbb{R}^6$ that maps the matrix $A$ to the vector $(a_x\ a_y\ a_z\ \alpha_x\ \alpha_y\ \alpha_z)^{\top}$. With this definition, the inverse operator \textls{vec}$^{-1}:\mathbb{R}^6\to\frak{se}(3)$ is uniquely determined.

To proceed, we construct the matrix  $J\in\mathbb{R}^{6\times6}$. Let $\theta=\|\alpha\|$ and define
{\small
\begin{eqnarray}
&& S_1 {=} \left(
\begin{array}{@{}ccc@{}}
0 & -\alpha_z & \alpha_y \\
\alpha_z & 0 & -\alpha_x \\
-\alpha_y & \alpha_x & 0
\end{array}
\right) \ , \
S_2 {=} \left(
\begin{array}{@{}ccc@{}}
0 & -a_z & a_y \\
a_z & 0 & -a_x \\
-a_y & a_x & 0
\end{array}
\right) \nonumber \\
&& S = \left(\begin{array}{cc} S_1 & S_2 \\ 0_{3\times3} & S_1 \end{array}\right) \nonumber \\
&& c_1 {=} \frac{4-\theta\sin(\theta)-4\cos(\theta)}{2\theta^2} \ , \
c_2 {=} \frac{4\theta-5\sin(\theta)+\theta\cos(\theta)}{2\theta^3} \nonumber \\
&& c_3 {=} \frac{2-\theta\sin(\theta)-2\cos(\theta)}{2\theta^4} \ , \
c_4 {=} \frac{2\theta-3\sin(\theta)+\theta\cos(\theta)}{2\theta^5}. \nonumber
\end{eqnarray}
}
Then
\begin{equation}
J = I_{6\times6} + c_1S + c_2S^2 + c_3S^3 + c_4S^4 .
\end{equation}

Next, note that the matrix $J$ has the following structure: there exist six vectors $a_i\in\mathbb{R}^3$ such that \cite{barfoot2017state}
\begin{equation}
J =
\left(
\begin{array}{cccccc}
a_1 & a_2 & a_3 & a_4 & a_5 & a_6 \\
0_{3\times1} & 0_{3\times1} & 0_{3\times1} & a_1 & a_2 & a_3
\end{array}
\right).
\end{equation}

Let $Q\in\textsl{SO}(3)$ and $d\in\mathbb{R}^3$ denote the rotation and translation components, respectively, of $\exp(-A)\in\textsl{SE}(3)$. Also let:
\begin{align} & b_1 = Qa_1 \ , \ b_2 = Qa_2 \ , \ b_3 = Qa_3 \nonumber \\ & b_4 = Qa_4 {-} b_1 {\times} d \ , \ b_5 = Qa_5 {-} b_2 {\times} d \ , \ b_6 = Qa_6 {-} b_3 {\times} d \ , \nonumber
\end{align}
where $\times$ denotes the cross product. Then the matrix $\mathbb{I}_A\in\mathbb{R}^{6\times6}$ can be obtained as
\begin{equation}
\mathbb{I}_A =
\left(
\begin{array}{cccccc}
b_1 & b_2 & b_3 & b_4 & b_5 & b_6 \\
0_{3\times1} & 0_{3\times1} & 0_{3\times1} & b_1 & b_2 & b_3
\end{array}
\right).
\end{equation}
Note that the diagonal block structure of the matrix $\mathbb{I}_A$ can be exploited to efficiently compute its inverse. 

\section{Experiments}

\subsection{Efficiency Analysis}

To validate the efficiency and accuracy of the analytical approximation, we conducted extensive testing in $G = \textsl{SE}(3)$, using the Piyavskii–Shubert algorithm as the baseline for finding the global minimum. Figure \ref{fig:topologies} illustrates the reference $G$-polynomial curves with varying orientation frames. 
\begin{figure*}[t]
    \centering
    \includegraphics[width=\textwidth]{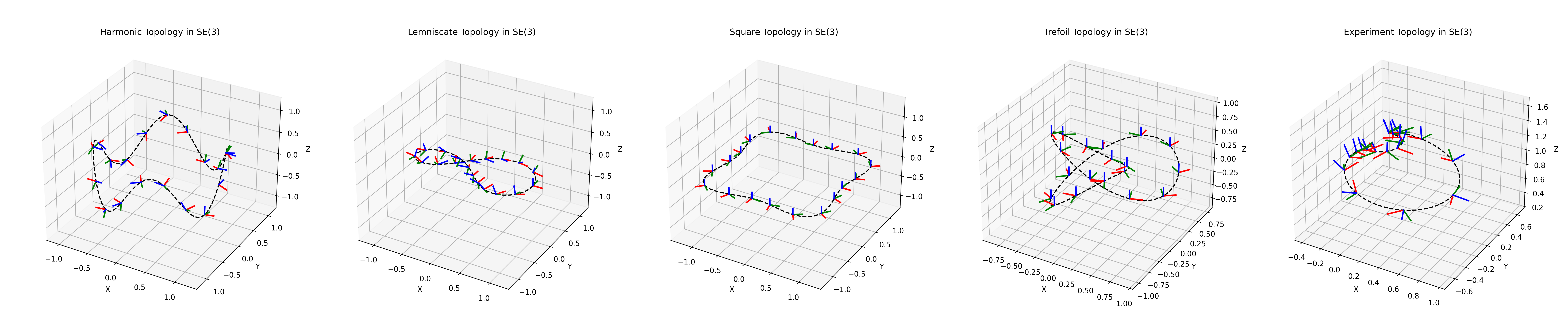}
    \caption{The baseline synthetic topologies in $\textsl{SE}(3)$ used for evaluation. The dashed line represents the translational path, while the RGB axes denote the orientation frames varying continuously along the parameter $s$.}
    \label{fig:topologies}
    \vspace{-0.6cm}
\end{figure*}

We constructed a comprehensive dataset originating from five base curves (named ``harmonic'', ``lemniscate'', ``square'', ``trefoil'' and ``experiment'' in Figure \ref{fig:topologies}). By systematically varying the frequency, scale, and shear of these topologies, we generated 150 unique base curves. Each curve was then discretized and interpolated into $G$-polynomials using an odd number of segments, $K \in \{13, 15, \dots, 41\}$. As mentioned in Subsection \ref{subs:approximateC}, this interpolation is linear in the number of segments $K$. From our experiments, it took, on average, $1.9$ $\mu s$ per segment to interpolate on a host computer equipped with an AMD Ryzen 3 7320U with Radeon Graphics @ 2.4 GHz CPU, and 8 GB of RAM.

Finally, for every configuration, we sampled $400$ test poses ($H \in \textsl{SE}(3)$), ultimately yielding exactly $150 \times 14 \times 400 = 840,000$ unique test pairs $(\mathcal{C},H)$.
For the baseline Piyavskii–Shubert algorithm, we established global optimality using a numerically estimated Lipschitz constant 
with a 5\% safety margin and a tolerance of $\delta = 10^{-4}$. Both algorithms were implemented in C++ and evaluated under 
identical environment.

We will now present the \emph{relative} computational time of our algorithm compared to the Piyavskii–Shubert method. The absolute computational time of our algorithm will be reported in the next subsection; nevertheless, it is also very fast, averaging on the order of $30 \mu s$ in average. Figure~\ref{fig:speedup_bands} illustrates the computational speedup ($x = \mbox{Time}_{\text{Baseline}}/\mbox{Time}_{\text{Our}}$) from $K$ up to $40$. 
The analytical method provides significant time reductions at lower values of $K$. 
As $K$ increases, the $G$-polynomial curve degenerates into a finely discretized point sequence, requiring more polynomial root 
evaluations and decreasing the relative speedup. Nevertheless, the analytical method maintains an average speedup 
ratio $> 1$ up to $K \approx 80$, a segment density well beyond typical operational requirements.
\begin{figure}[htpb]
    \centering
    \includegraphics[trim=3.8cm 0cm 3.8cm 1cm, clip, width=0.9\columnwidth]{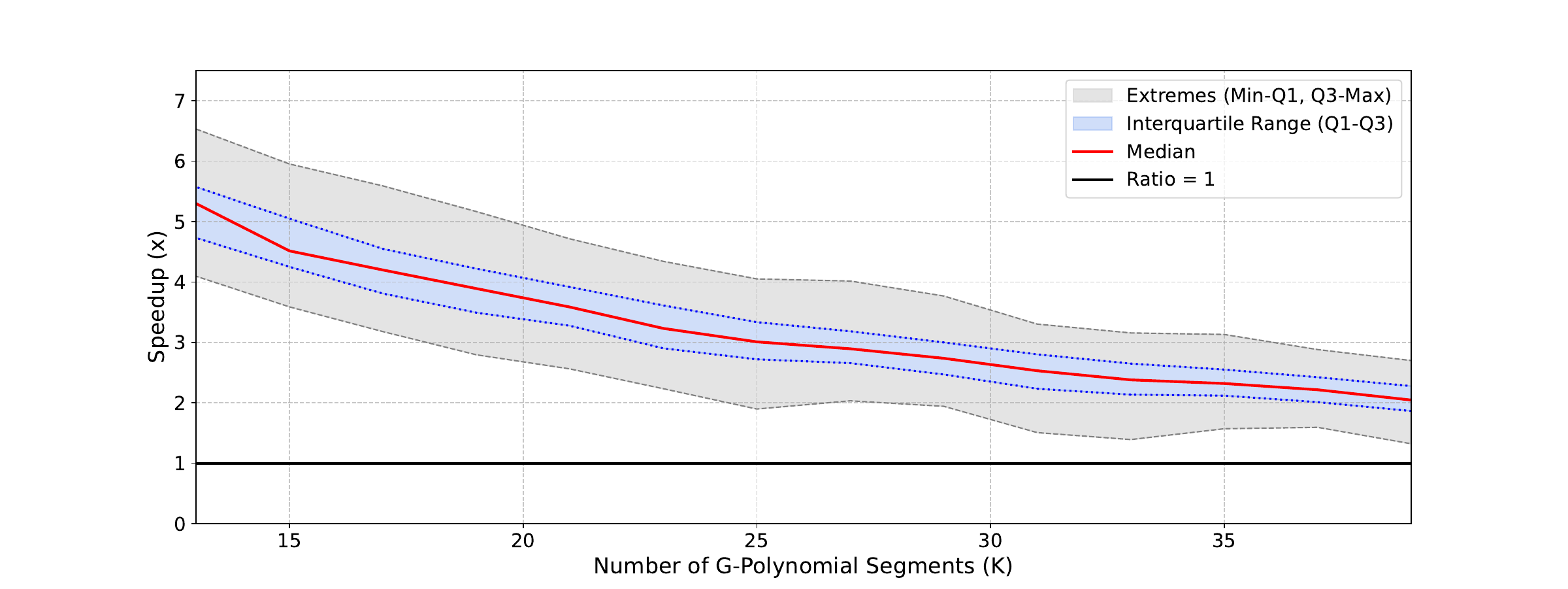}
    \caption{Distribution of analytical speedup by the number of $G$-polynomial segments ($K$).}
    \label{fig:speedup_bands}
\end{figure}

Computational performance is also dependent on the spatial proximity of the target pose to the curve. 
Figure \ref{fig:speedup_distance} demonstrates the interquartile range of the speedup as a function of the true distance $D(H)$. For poses very close to the curve ($D(H) < 0.06$), the baseline becomes competitive. This happens because the Piyavskii–Shubert's mechanism rapidly 
prunes the search space when the global minimum is small, which happens very close to the curve (where the global minimum is $0$). Conversely, for poses not very close, the analytical approach offers a dominant, stable advantage, as its computational complexity remains bounded by the fixed segment count regardless of $D(H)$. It is also important to note that this negative proximity effect on the speedup becomes less pronounced as the number of segments $K$ decreases. Nevertheless, we can adopt a hybrid approach and switch to the Piyavskii–Shubert method when we are in regions very close to the curve.

\begin{figure}[htpb]
\centering
\includegraphics[trim=3.2cm 0cm 3.4cm 0cm, clip, width=0.9\columnwidth]{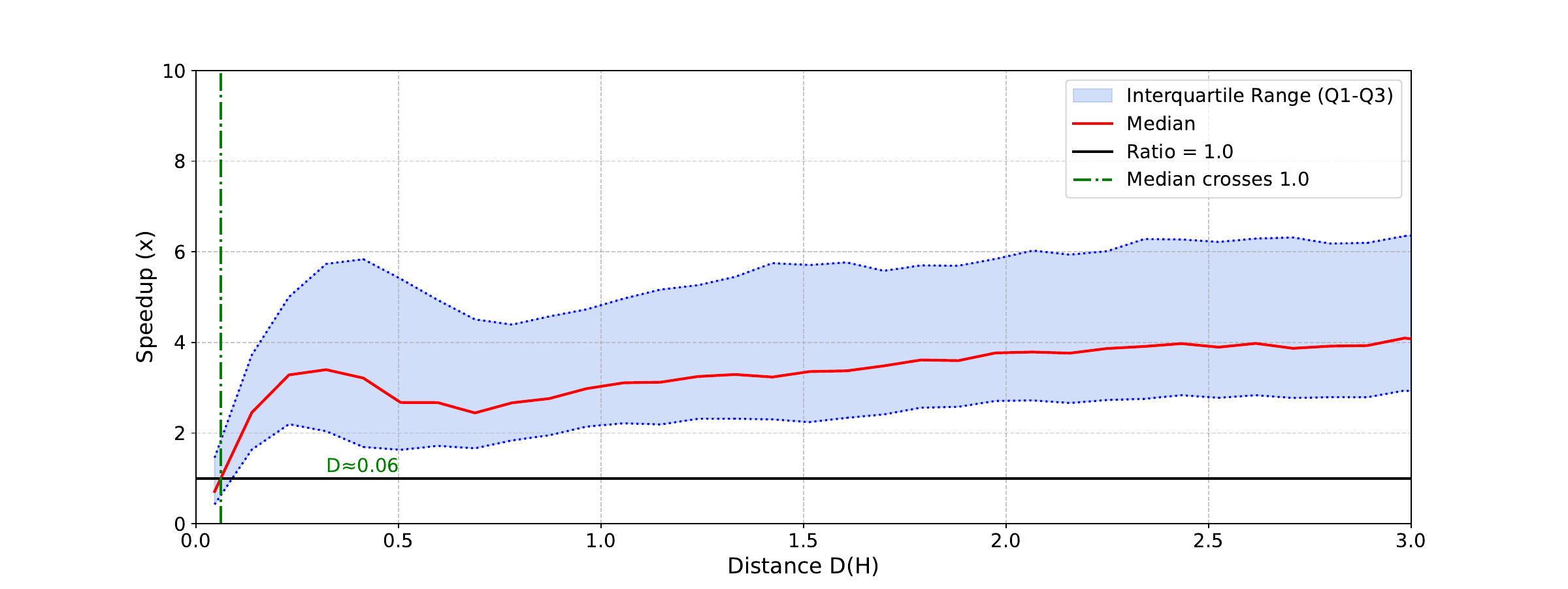}
\caption{Interquartile range of the computational speedup relative to the spatial distance $D(H)$. The last quartile, which can exhibit very large speedups, is not shown so as not to distort the graph.}
\label{fig:speedup_distance}
\vspace{-0.2cm}
\end{figure}

Because the proposed method relies on a first-order approximation, directly evaluating the polynomial roots yields a fast but sub-optimal initial guess $s_{best}$. To eliminate this theoretical difference gap, a few steps of a localized derivative descent step, as explained in Subsection \ref{subs:alg}, are applied to drive the parameter to the true local minimum, which, hopefully, is going also to be global.

We evaluated the error in the parameter domain between the minimizer yielded by the analytical approximation ($\tilde{s}$) 
and the Piyavskii–Shubert minimizer ($s_{ori}$). Because the parameter $s$ is normalized in the periodic interval 
$[0, K]$ to represent closed paths, the error was calculated as the shortest distance on the periodic domain: 
$\min(|\tilde{s} - s_{ori}|/K, 1 - |\tilde{s} - s_{ori}|/K)$. The approximation proved highly accurate: only 0.605\% of the total samples  exhibited errors above 1\%. As defined in Sec~\ref{subs:validity}, these approximation degradations occurs mainly when we are close to a point $H$ when the curve has more than one equidistant minimums or when the test pose approaches singularities (where $H^{-1}C_k \notin G^{\circ}$).

\subsection{Real-World Kinova Gen3 Experiment}

To demonstrate the practical viability of the proposed analytical distance formulation in real-time robotic applications, 
we implemented the vector field algorithm in \cite{bartelt2026constructive}, which requires the distance computation, for closed-loop path tracking experiment using a physical 7-DOF Kinova Gen3 manipulator (see Fig.~\ref{fig:real_experiment_snapshots} and the video in the supplementary material). 
The target curve $\mathcal{C}$ was taken like the last topology in Figure \ref{fig:topologies}  with  $K=17$ 
segments.

\begin{figure*}[htpb]
    \centering
    \includegraphics[width=0.48\textwidth]{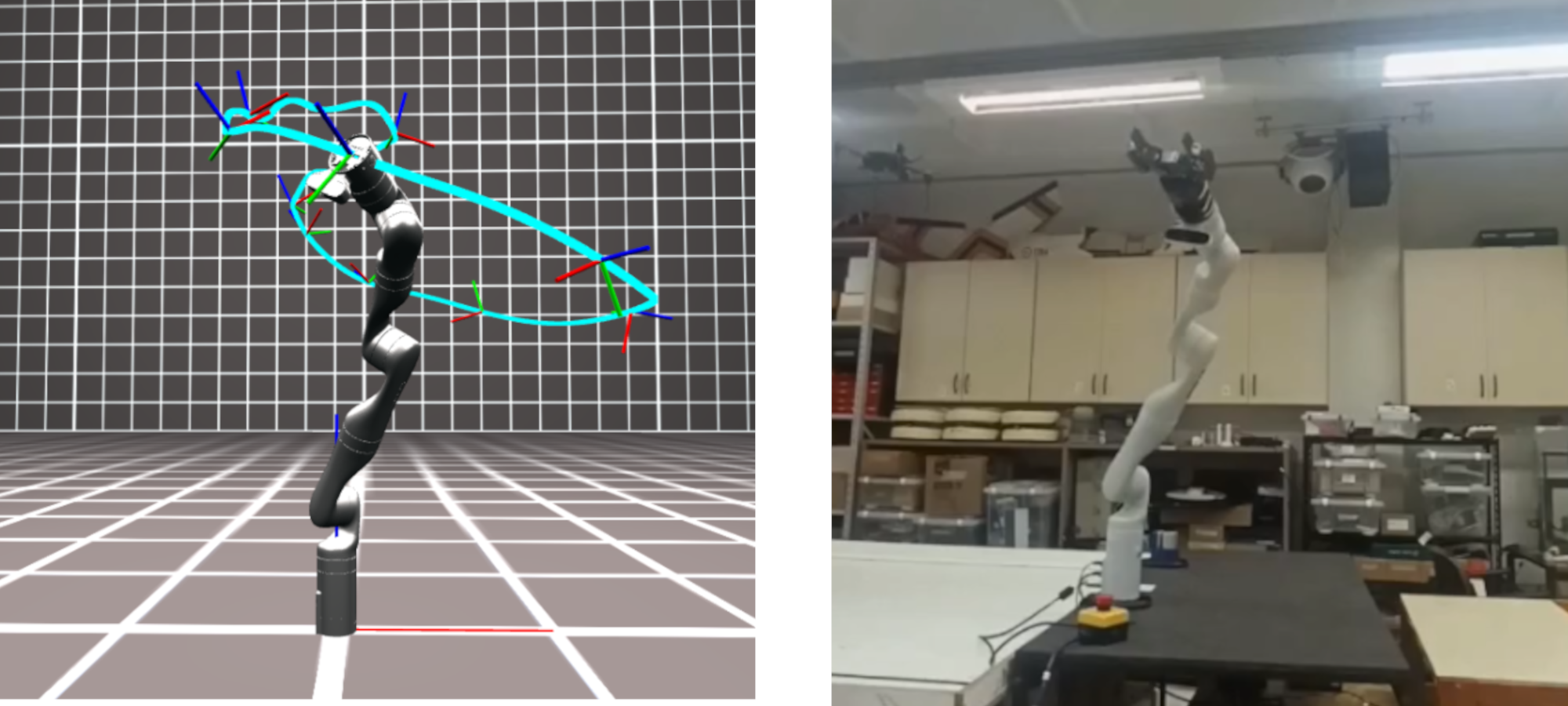}
    ~
    \includegraphics[width=0.48\textwidth]{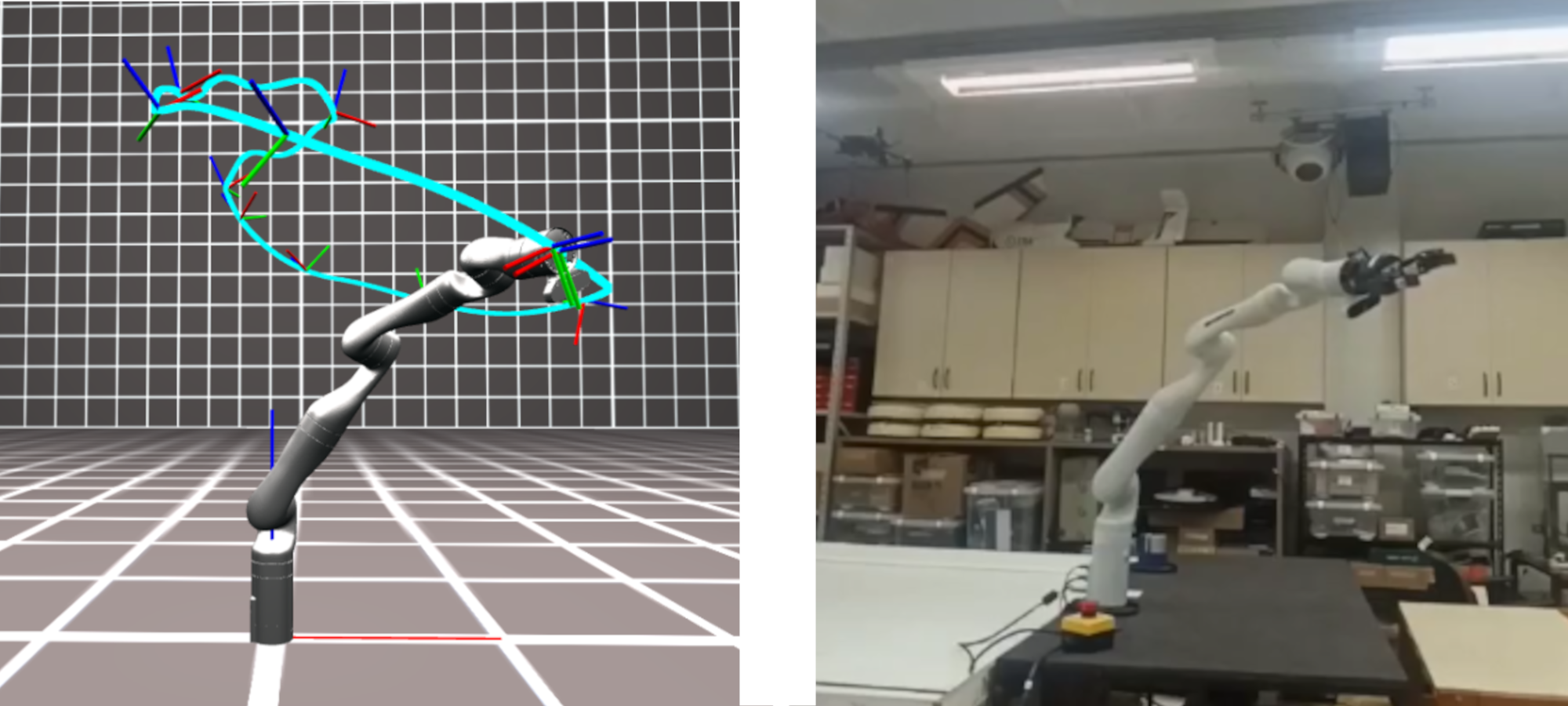}
    \caption{Snapshots of the experiment, featuring the reconstruction of the manipulator's configuration and the interpolated target curve (left), paired with the corresponding real-world view of the experiment (right).}
    \label{fig:real_experiment_snapshots}
    \vspace{-0.6cm}
\end{figure*}

The control architecture was deployed on a host computer equipped with an Intel i5-10300H @ 4.5 GHz CPU, with 8 GB of RAM. 
The host machine communicated with the robot's low-level controller via a UDP/IP Ethernet connection. 
The robot was driven in velocity control mode at a frequency of 100 Hz ($\Delta t = 0.01 \text{ s}$). 
At each control iteration, the current end-effector pose $H \in \textsl{SE}(3)$ was acquired via forward kinematics 
based on the measured joint angles $q$. The analytical solver was then queried to evaluate the spatial distance $D(H)$ 
and compute the corresponding corrective twist $\xi \in \mathbb{R}^6$ generated by the vector field. Finally, 
this task-space twist was mapped to joint velocity commands $\dot{q}$ using inverse differential kinematics.

 Over the duration 
of the experiment, the vector field solver was queried exactly $131,761$ times. The measured average computation time 
per query was $32.20\ \mu\text{s}$, with a standard deviation of $11.46\ \mu\text{s}$.

\section{Conclusion and Future Work}
\label{sec:conclusion}

Vector field navigation in matrix Lie groups requires continuously computing the distance between a robot's configuration and a 
target curve, which is typically a computationally expensive task. This paper introduced an efficient analytical approach to 
approximate this distance by representing the target path as a $G$-polynomial curve with $K$ segments. By exploiting this structure, the method 
reduces the optimization problem to a small number of polynomial root-finding computations.

We presented practical formulae for the case in which the group is $G = \textsl{SE}(3)$. Extensive numerical evaluations in this setting demonstrated that the proposed formulation yields significant computational 
speedups for several segments $K \le 80$ when compared to the most suitable alternative approach, the Piyavskii–Shubert algorithm.

Our analysis showed that the baseline algorithm remains 
competitive for poses very close to the curve due to its rapid search space pruning. Therefore, 
future work will focus on refining the analytical method to further improve its computational efficiency in these situations. This can be done by, instead of checking the distance $D(H)$ for all segments, pruning some segments by computing a much quicker distance proxy $\tilde{D}(H)$. Another alternative is to apply the Piyavskii–Shubert for a few steps to bound the interval in which the optimal solution lies on and then apply our methodology only on the segments that overlap with this interval.

\bibliographystyle{ieeetr}  
\bibliography{bibl}
\end{document}